\def\year{2017}\relax
\documentclass[letterpaper]{article}
\include{Definitions}
\usepackage{aaai17}
\usepackage{times}
\usepackage{helvet}
\usepackage{courier}
\usepackage[usenames,dvipsnames]{color}
\usepackage[linesnumbered, ruled, vlined, noend]{algorithm2e}
\usepackage{amsmath}
\usepackage{amsfonts}
\usepackage{bm}
\usepackage{booktabs}
\usepackage{caption}
\usepackage{couriers}
\usepackage{dsfont}
\usepackage[shortlabels]{enumitem}
\usepackage{mathtools}
\usepackage{nicefrac}
\usepackage{natbib}
\usepackage{graphicx}
\usepackage{url}
\usepackage[caption=false,font=footnotesize]{subfig}
\graphicspath{ {Figures/} }  
\DeclareGraphicsExtensions{.pdf,.jpeg,.png}

\newcommand{\mathfont}[1] {#1}
\newcommand{\displayfont}[1] {{\fontsize{8pt}{9pt}\selectfont #1}}

\newcommand{\inlinefont}[1] {{\fontsize{8.5pt}{9pt}\selectfont #1}}

\newcommand{\true}{global\xspace}

\newcommand{\pr}{\mathrm{Pr}}

\newcommand{\bw}{{\bf w}}

\newcommand{\cS}{\mathcal{S}}

\begin{document}
\title{Stochastic Gradient Descent for Relational Logistic Regression via \\
Partial Network Crawls}
\author{
Jiasen Yang{\footnotemark[1]}
\quad Bruno Ribeiro{\footnotemark[2]}
\quad Jennifer Neville{\footnotemark[1]\footnotemark[2]} \\
Departments of Statistics\footnotemark[1] and
Computer Science\footnotemark[2] \\
Purdue University, West Lafayette, IN \\
\url{{jiaseny, ribeirob, neville}@purdue.edu} \\
}

\maketitle


\begin{abstract}

Research in statistical relational learning has produced a number of methods for learning relational models from large-scale network data.
While these methods have been successfully applied in various domains,
they have been developed under the unrealistic assumption of full data access.
In practice, however, the data are often collected by crawling the network,
due to proprietary access, limited resources, and privacy concerns.
Recently, we showed that the parameter estimates for
relational Bayes classifiers computed from network samples collected by existing network crawlers can be quite inaccurate, and developed a \emph{crawl-aware} estimation method for such models \citep{YanRibNev17}.
In this work, we extend the methodology to learning relational logistic
regression models via stochastic gradient descent from partial network crawls,
and show that the proposed method yields accurate parameter estimates and confidence intervals.

\end{abstract}


\section{Introduction}\label{sec:intro}

There has been a great deal of interest in learning statistical models
that can represent and reason about {\em relational} dependencies (see \eg , \citeauthor{GetoorTaskar07}, \citeyear{GetoorTaskar07}).
For example, political views are often correlated among friends in social networks.
While much work has been done
in the relational learning community to develop models and algorithms for estimation and inference in networks, a primary assumption underlying these works is that a {\em full} network is available for learning.
With access to the full network, one could perform \emph{stochastic gradient descent} (SGD) in the usual manner, and the learned parameters will
asymptotically converge to the desired parameter estimates~\citep{Robbins1951,bottou2005line,bach2014adaptivity}.

However, the network datasets used to study relational models are typically {\em samples} of a larger network.
In particular, it is often the case that researchers do not have
\emph{random access} to the full network structure and that
sampling is only possible via repeated crawling from a node to
its neighbors---a procedure that tends to result in biased samples~\citep{kurant2011towards,Ribeiro2012}.
As a result, naively performing SGD using these partial crawls may also
suffer from unknown biases.
Part of this work shows that such crawled data could indeed lead to biased
parameter estimates in real-world scenarios.

Recently, we showed that estimating the parameters in a
\emph{relational Bayes classifier} (RBC) \citep{NevJenGal03,MacPro07}
using data from widely used network sampling
methods---such as \emph{snowball sampling}, \emph{forest-fire} \citep{LesFal06}, \emph{random walks}, and \emph{Metropolis-Hastings random walk} \citep{GjoKurButMar09}---could lead to biased parameter estimates \citep{YanRibNev17}.
We then corrected for such estimation bias in the RBC by exploiting a general crawling method introduced by \cite{AvrRibSre15} that produces unbiased estimates with statistical guarantees.

In this work, we extend the methodology to
develop a crawl-based SGD procedure for \emph{relational logistic regression} (RLR).
The proposed method is guaranteed to obtain unbiased estimates of the log-likelihood function and its gradients over the full network (with finite variance), which allows SGD to converge to the correct parameter values for sufficiently small learning rates~\citep{Robbins1951}.
Furthermore, we show how to construct confidence intervals of the estimated parameters,
which enables practitioners to assess the statistical significance of
features in the model.

\vspace{-1pt}
\subsubsection*{Summary of Contributions}
\vspace{-2pt}

\begin{itemize}[leftmargin=*]
\itemsep0em
\item We derive a crawl-based SGD method for learning the RLR model from partial
network crawls, and prove that the proposed method yields unbiased estimates
of the log-likelihood function and its gradients over the full network.
\item We demonstrate how to construct confidence intervals of the estimated parameters by exploiting the independence properties of the network samples.
\item We conduct experiments on several large network datasets, and demonstrate
that the proposed methodology achieves consistently lower error in parameter
estimates and higher coverage probabilities of confidence intervals.
\end{itemize}


\section{Problem Definition}\label{sec:problem}

The goal of this work is to develop a stochastic estimation algorithm for
the relational logistic regression (RLR) model in large social networks under an access-restricted scenario.
In particular, we are interested in accurately estimating model parameters in order to effectively assess the importance of relational features involving the neighbors of a node.%
\footnote{In \cite{YanRibNev17}, we also referred to the parameters in a relational model
as {\em peer effects}.}

We assume that random access to the full network structure is not
available, and that the network can only be accessed via crawling.
Specifically, we assume (\emph{i}) the availability of a seed node in the network, (\emph{ii}) the ability to query for the attributes of a sampled node, and (\emph{iii}) the ability to transition to neighbors of a sampled node.

Given such an access pattern,
and assuming that the full network cannot be crawled,
the task is to accurately estimate the model parameters by learning an RLR model
over the sampled network.
If we refer to the estimates that a learning algorithm would obtain from the
full network as \emph{\true} estimates and those from the sampled network as
\emph{sample} estimates, then the ideal method
should produce
(\emph{i}) unbiased sample estimates (\wrt the \true estimates), and
(\emph{ii}) accurate assessments of the uncertainty associated with the sample estimates (\eg, confidence intervals).


\section{Background and Related Work}\label{sec:background}
Denote a graph
by $G = (V, E)$, where $V$ is the set of vertices and $E\subseteq V\times V$ is the set of edges.
For a node $v\in V$, denote its neighbors by $\Ncal_v = \{u\in V: (u, v)\in E\}$ and its degree by $d_v = |\mathcal{N}_v|$.
Finally, let $\mathds{1}\{\cdot\}$ denote the indicator function.

\subsubsection{Network Sampling Algorithms}\label{sec:background_sampling}

We note that under a crawl-based scenario,
any technique involving random node/edge selection
will be infeasible.

\emph{Snowball sampling (BFS)}
traverses the network via a breadth-first search.
\emph{Forest fire (FF)} \citep{LesFal06}
samples (``burns'') a random fraction of a node's neighbors,
and repeats this process recursively for each ``burned'' neighbor.
\emph{Random walk sampling (RW)}
performs a random walk on the network by transitioning
from the current node to a randomly selected neighbor at every step.
%
\emph{Metropolis-Hastings random walk (MH)}
\citep{GjoKurButMar09}
sets the transition probability from node $u$ to $v$ as
$\Pvec_{u, v}^\mathsf{MH} = \min(\nicefrac{1}{d_u}, \nicefrac{1}{d_v})$
if $v \in \Ncal_u$ and
$1 - \sum_{w\neq u}\Pvec_{u, w}^\mathsf{MH}$
if $v = u$, which yields a uniform stationary distribution over nodes.

\emph{Random walk tour sampling (TS)} \citep{AvrRibSre15}
is a recently proposed method that exploits the regenerative properties of
random walks.
Given an initial seed node, the algorithm first performs a short random walk
to collect a set of seed nodes $\Scal \subseteq V$,
and then proceeds to sample a sequence of random walk \emph{tours}.
Specifically, the $k$-th random walk tour starts from a sampled node
\inlinefont{$v_1^{(k)}\in \Scal$} and transitions through a sequence of nodes \inlinefont{$v_2^{(k)},\,\dots,v_{\xi_k-1}^{(k)}\in V\backslash \Scal$} until it returns to a node
\inlinefont{$v_{\xi_k}^{(k)} \in \Scal$}.
The algorithm repeats this process to sample $m$ such tours, denoted
\inlinefont{$\Dcal_m (\Scal) = \{(v_1^{(k)},\ldots,v_{\xi_k}^{(k)})\}_{k=1}^m$}.
Since the successive returns to a seed node in $\Scal$ act as renewal epochs,
the renewal reward theorem~\citep{Bremaud99} guarantees that sample statistics computed
from each tour will be independent.

\subsubsection{Relational Learning Models} \label{sec:relmodels}

Relational learning models (see \eg, \citeauthor{GetoorTaskar07}, \citeyear{GetoorTaskar07}) extend traditional supervised learning methods to the relational domain,
in which training examples (such as nodes in a social network) are no longer
\iid~(independent and identically distributed).

\emph{Relational logistic regression (RLR)} (see \eg, \citeauthor{KazBucKerNat14}, \citeyear{KazBucKerNat14})
predicts the target class of a node using \emph{aggregated} features
constructed from the class label and attributes of its neighbors.
A typical aggregation function is \emph{proportion}, which takes the proportion
of neighbors that possess a particular class label or feature value.
Let $\phi_v\in\RR^d$ be a set of aggregated features for node $v \in V$ that
involve either the attributes of $v$ or the attributes/class label of its neighbors $\mathcal{N}_v$.
Let $y_v \in\{1, \dots, H\}$ be the class label of node~$v$.
The RLR model employs the \emph{soft-max} function to predict~$y_v$:

\vspace{-2mm}%
\displayfont{
\begin{align}
\nonumber \pr( y_v | \bw_1,\ldots,\bw_H, \phivec_v ) = \frac{ \exp\left( \bw^\ts_{y_v} \phivec_v \right)}{ \sum_{h=1}^H \exp(\bw^\ts_h \phivec_v)}\, ,
\end{align}
}%
where $\bw_c\in\RR^d$ are the weights for class $c \in \{1, \dots, H\}$
that need to be estimated from the network.

\subsubsection{Related Work}

In our previous work \citep{YanRibNev17}, we showed that the class priors and conditional probability distributions (CPDs) in a relational Bayes classifier (RBC) can be
unbiasedly estimated under the same crawl-based scenario that we consider here.
In this work, we extend the methodology to the estimation of RLR models using a crawl-based SGD method.
Note that RLR forms a more expressive model family which poses a more challenging estimation task---in fact, the CPDs in an RBC can be implicitly represented by features in an RLR model.
Furthermore, statistical significance tests (\eg, $\chi^2$ and deviance tests)  for the parameter estimates in an RLR model have been extensively studied
(see \eg, \citeauthor{Agresti02}, \citeyear{Agresti02}) in the literature,
which offer tools for feature selection and model comparison.


\section{Proposed Methodology}\label{sec:methods}

Given an unobserved network $G = (V, E)$,
the tasks are (\emph{i}) to
estimate the parameters $\thetab$ in a relational model
by crawling the network $G$ from an initial set of seed nodes $\Scal\subseteq V$,
and
(\emph{ii}) to assess the uncertainty associated with the estimates~$\widehat{\thetab}$.
To this end, we outlined the following procedure for crawl-based estimation of relational models in large-scale networks \citep{YanRibNev17}:
\begin{description}[font=\normalfont\emph]
    \itemsep=0pt
    \item[Crawling] Crawl the network using a sampling method.
    \item[Estimation] Estimate parameters from the crawled network.
    \item[Calibration] Compute confidence intervals for the estimates.
\end{description}

For the crawling phase, we shall employ the random walk tour sampling algorithm
(see Section~\ref{sec:background}).
Next, we discuss the details of our proposed stochastic estimation and calibration methodology for relational logistic regression (RLR).

\subsubsection{Relational Model Estimation}\label{sec:srl}

Recall that RLR utilizes a multinomial logistic regression model to define conditional probability of the label of node $v \in V$ given its observed attributes and the attributes/class label of its neighbors $\mathcal{N}_v$.
Let $\phivec_v\in\RR^d$ be a set of aggregated features for node $v \in V$ that
is computed from its neighbors $\mathcal{N}_v$.
In practice, to reduce the node-querying cost, we can estimate the
aggregated features $\phivec_v$ stochastically by taking a uniform
sample of the neighbors $\Ncal_v$.
The log-likelihood for the RLR model is given by

\vspace{-2mm}%
\displayfont{
\begin{align}
\nonumber \Lcal(\bw_1,\ldots,\bw_H)
& = \sum_{v \in V} \log \pr(y_v | \bw_1,\ldots,\bw_H, \phivec_v) \\
& = \sum_{v \in V} \left[ \bw^\ts_{y_v} \phivec_v - \log\left( \sum_{h=1}^H \exp(\bw^\ts_h \phivec_v) \right) \right], \label{eq:rlr-loglik}
\end{align}
}%
where $\bw_c\in\RR^d$ are the weights for class $c \in \{1, \cdots, H\}$.

In general, we do not have any guarantees on the quality of the parameter
estimates obtained from a crawled network.
However, if the sample were collected using the tour sampling algorithm,
we propose applying Theorem~\ref{thm:rlr} to accurately estimate the parameters in an RLR model.

\begin{theorem}[RLR crawl-based SGD]
\label{thm:rlr}
Given the sampled random walk tours
\mathfont{$\Dcal_m (\Scal) = \{(v_1^{(k)},\ldots,v_{\xi_k}^{(k)})\}_{k=1}^m$},
the following estimates for the log-likelihood of Eq.\,\eqref{eq:rlr-loglik}
and its gradients are unbiased:
%
\displayfont{
\begin{align}
\widehat{\Lcal} (\bw_1,\ldots,\bw_H)
& \defeq \frac{d_{\Scal}}{m}  \sum_{k=1}^m
 \sum_{t=2}^{\xi_k-1} \frac{g(v^{(k)}_t)}{d_{v^{(k)}_t}}
+
\sum_{v \in \cS} g(v) \label{eq:RLR-loglik} \\
\nabla_{\bw_j} \widehat{\Lcal} (\bw_1,\ldots,\bw_H)
& \defeq \frac{d_{\Scal} }{m} \sum_{k=1}^m
\sum_{t=2}^{\xi_k-1} \frac{g'_j(v^{(k)}_t)}{d_{v^{(k)}_t}}
+
\sum_{v \in \cS} g'_j(v)
 \label{eq:RLR-loglikgrad} \, ,{}
\end{align}
}%
where
$d_{\Scal} = |((\Scal\times V) \cap E) \backslash (\Scal\times\Scal)|$ denotes
the total number of outgoing edges from the seed nodes, and

\vspace{-3mm}%
\displayfont{
\begin{align*}
g(v)
& \defeq \bw^\ts_{y_v} \phivec_v - \log\left( \sum_{h=1}^H \exp(\bw^\ts_h \phivec_v) \right) \\
g'_j(v)
& \defeq \left( \mathds{1}\{y_{v} = j\} - \frac{\exp(\bw_j^\ts \phivec_{v})}{\sum_{h=1}^H \exp(\bw_h^\ts \phivec_{v})} \right) \phivec_{v}.
\end{align*}
}%
\end{theorem}
\begin{proof}
We defer the proof to the Appendix.
\end{proof}
We can learn the weights of an RLR model by minimizing the negative
log-likelihood via stochastic gradient descent using the
estimates of Eq.\,\eqref{eq:RLR-loglikgrad}.
While Theorem~\ref{thm:rlr} shows that the log-likelihood over the full network
and its gradients can be unbiasedly estimated, we note that the current result
does not directly imply bounds on the approximation error of the resulting parameter estimates.
Empirically, our experiments suggest that the parameter estimates are quite accurate.

\subsubsection{Calibration of Estimated Parameters}\label{sec:calibration}

To construct
confidence intervals for the estimated parameters,
we propose to utilize bootstrap resampling \citep  {Efron79}.
In fact, this step can be performed as sampling
progresses---by monitoring the confidence intervals, the
practitioner can determine adaptively if more samples need to be collected.

For tours sampling, since the estimates computed from each tour 
are independent,
we can perform bootstrapping by treating the tours individually,
sample with replacement, compute an estimate over the bootstrap sample, and
repeat this process.
We can then compute empirical confidence intervals over the bootstrap estimates.
Algorithm~\ref{alg:bootstrap} describes the bootstrapping algorithm we use
to compute confidence intervals for the parameters $\thetab$ in a general relational model.
For convenience, denote the nodes sampled in the $k$-th tour by
\mathfont{$\Tcal_k \defeq \{v_t^{(k)}\}_{t=1}^{\xi_k}$}.
\begin{algorithm}[!htbp]
  \caption{Computation of Confidence Intervals}\label{alg:bootstrap}
  \KwIn{The sampled tours
    \mathfont{$\Dcal_m (\Scal) = \big\{\Tcal_k\big\}_{k=1}^m$}; and \\
    \qquad\quad the number of bootstrap samples $B$.}
  \KwOut{A $100\,(1-\alpha)\%$ confidence interval for $\theta$.}

  $\Theta \leftarrow \varnothing$

  \For{$i = 1,2,\cdots,B$} {
    $\Dcal^{(i)}_m \leftarrow \varnothing$ \\
    \For{$j = 1, 2, \cdots, m$} {
        $k \leftarrow$ Random integer in $\{1, 2, \cdots, m\}$ \\
        $\Dcal^{(i)}_m \leftarrow \Dcal^{(i)}_m \cup \{\Tcal_j\}$
    }
    $\widehat{\theta}_i \leftarrow$ Estimate of $\theta$
        computed using $\Dcal^{(i)}_m$ \\
    $\Theta \leftarrow \Theta \cup \{ \widehat{\theta}_i \} $
  }
  $Q_{\nicefrac{\alpha}{2}} \leftarrow$ The $100\,(\nicefrac{\alpha}{2})$-percentile of $\Theta$ \\
  $Q_{1-\nicefrac{\alpha}{2}} \leftarrow$ The $100\,(1-\nicefrac{\alpha}{2})$-percentile of $\Theta$ \\
 \Return{$(Q_{\nicefrac{\alpha}{2}}, Q_{1-\nicefrac{\alpha}{2}})$}
\end{algorithm}
Among all the crawling methods under examination,
tours sampling is the only approach capable of producing
theoretically justified confidence intervals via resampling.
This is due to the fact that BFS, FF, RW, and MH do not provide a list of
node/edge samples that yield \iid estimates of the model parameters.


\section{Experimental Evaluation}\label{sec:exp}

\subsubsection{Dataset Description}

We perform experiments on five different attributed network datasets.
As a preprocessing step, we take the giant component of all networks.
Table~\ref{tab:friendster} shows the summary statistics for each
network after processing.

\begin{table*}[!htb]
\centering
\caption{Summary of Network Statistics}
\vspace{-8pt}
\fontsize{8pt}{10pt}\selectfont
\label{tab:friendster}
\begin{tabular}{rrrp{1.2cm}p{9.8cm}}
\toprule
Dataset    & $|V|$ & $|E|$ & Attributes & Class Prior Distribution \\
\midrule
Facebook   &  14,643 & 336,034 & \emph{PoliticalView} &   Conservative: 28.40\%, Otherwise: 71.60\%                 \\
Friendster-Large & 3,146,011 & 47,660,702 & \emph{Age} & [16, 26): 35.02\%, [26, 28): 16.27\%,
                                  [28, 32): 22.56\%, [32: 100): 26.15\%  \\
           &  &  & \emph{Gender} & Female: 46.99\%, Male: 53.01\% \\
           &  &  & \emph{Status} & Single: 67.21\%, In a Relationship: 19.06\%, Married: 12.60\%, Domestic Partner: 1.13\% \\
Friendster-Small & 1,120,930 & 19,342,990 & \emph{Age}  & [16, 26): 44.53\%, [26, 28): 14.96\%, [28, 32): 21.28\%, [32, 100): 19.23\%  \\
           &  &  & \emph{Gender} & Female: 45.39\%, Male: 54.61\% \\
           &  &  & \emph{Status} & Single: 62.01\%, In a Relationship: 20.29\%,
                                   Married: 16.50\%, Domestic Partner: 1.21\% \\
           &  &  & \emph{Zodiac} & Capricorn: 7.74\%, Virgo: 8.27\%, Libra: 8.60\%,
                                   Gemini: 8.11\%, \\
           &  &  &               & Scorpio: 8.31\%, Leo: 8.88\%, Taurus: 8.79\%,
                                   Sagittarius: 8.58\%, \\
           &  &  &               & Cancer: 8.15\%, Aquarius: 8.97\%, Pisces: 7.83\%, Aries: 7.77\% \\
Communications  & 855,172 & 5,269,278 & \emph{Comm.} & Yes: 6.09\%, No: 93.91\% \\
Computers  & 855,172 & 5,269,278 & \emph{Computers} & Yes: 17.34\%, No: 82.66\% \\
\bottomrule
\end{tabular}
\vspace{-10pt}
\end{table*}

\emph{Facebook} is a snapshot of the Purdue University
Facebook network consisting of users who have listed their political views
(whether or not they declare to be conservative).

\emph{Friendster-Large (Fri.-L.)} and \emph{Friendster-Small (Fri.-S.)} are
processed from
the entire Friendster social network crawl \citep{MouNaiRibNev17}.
For \mbox{Fri.-L.}, we take the subgraph containing all users with
\emph{age}, \emph{gender}, and \emph{marital status} listed in their profiles.
For Fri.-S., we also include \emph{zodiac}.
We discretized the \emph{age} attribute into four interval classes.

The observations we make are not restricted to social networks.
We also experiment on two citation networks,
\emph{Communications (Comm.)} and \emph{Computers}, both
constructed from the NBER patent citations dataset \citep{HalJafTra01}.%
\footnote{While the edges in the citation networks are directed,
we treat them as undirected edges in the experiments for simplicity.}
The label of each patent indicates whether it was filed in
a category related to comm.~(computers).

\subsubsection{Experiment Setup}

In each run of the simulation, we randomly select 50\%
nodes in the network to have observed labels, and the task is to infer the labels of the remaining nodes.
Next, a labeled node is randomly selected 
as the seed node to initiate crawling for all the sampling methods.%
\footnote{In practice, one could always avoid querying unlabeled
nodes; thus, we set all methods to crawl directly on the labeled subgraph.}
In practice, querying a node will be associated with a certain cost, and we
strictly control for the number of unique node-queries.
For each method, we keep track of the parameter estimates and bootstrap
confidence intervals as crawling progresses.
%
We perform 10 runs of the simulation,
and report the average performance and standard errors for all methods.

\subsubsection{Evaluation Criteria}
We measure the performance of the various network crawling methods
in terms of:
\begin{itemize}
  \itemsep=0pt
  \item The quality of the RLR parameter estimates learned from a network
    sample crawled  using that method. Specifically, we measure
    (\emph{i}) the mean-absolute-error (MAE) between
            the \emph{sample} estimate computed from the crawled sample and
            the \emph{global} estimate computed from the entire graph, and %
    (\emph{ii})
            the root-mean-square-error (RMSE) of the predicted class
            probabilities for the unlabeled nodes%
            \footnote{
                When predicting the class label for an unlabeled node, in addition to
                the attributes and class label of its neighbors, the attributes
                (but not the class label) of the unlabeled node
                are also available.}
            using an RLR model equipped with the \emph{sample} estimates.
  \item The quality of the confidence intervals obtained from the
    crawled sample, as measured by the coverage probability.
\end{itemize}

\subsubsection{Evaluation of Estimation Performance}

Figure~\ref{fig:estimation-rlr} shows the quality of the estimated parameters
vs. the proportion of queried nodes as crawling progresses.%
\footnote{For the Friendster results, the parenthesized attribute
denotes the class label used for the prediction task, while all other attributes are used as features.
The solid line in the RMSE plots correspond to the prediction error obtained
using the \emph{global} estimates.
The plots are jittered horizontally to prevent the error bars from overlapping.}
We observe that across all datasets,
tour sampling (TS) consistently
achieves smaller MAE in the estimated model parameters as well as lower RMSE in the
predicted class probabilities.
Also note that MH and RW usually outperfom FF and BFS.

For numerical stability reasons, we utilized either $\ell_1$ or $\ell_2$ regularization in our experiments.
Since our interest is in accurately estimating model parameters, we set the regularization parameter to be very small~($10^{-3}$) in both cases.
Not surprisingly, $\ell_1$ regularization results in sparser parameter estimates.
Also note that the unbiased estimators we proposed for TS automatically
scales up the estimate of the log-likelihood and its gradients to match that of
the full-network, whereas those obtained using conventional sampling
methods would depend on the size of the crawled sample.

Comparing the learning curves for the parameter estimates MAE with those of the
RMSE of predicted class probabilities as well as the classification accuracy,
we notice that the required sample size to achieve
reasonable prediction performance can be much smaller than the sample size
required to accurately estimate model parameters.
Figure~\ref{fig:estimation-acc} also shows that
more accurate parameter estimates do not necessarily translate to improved
classification accuracy, as in some cases biased parameter estimates may lead to better generalization performance.

\subsubsection{Evaluation of Calibration Performance}

Figure~\ref{fig:variance-rlr} 
shows the estimated bootstrap sampling distributions 
for several model parameters.
\footnote{For BFS, FF, RW, and MH, we perform bootstrapping directly on the
sampled nodes by treating each node as an \iid instance.}
We observe that TS is the only method consistently capturing the \true
parameter values.

To assess the calibration performance of each method,
we compute 95\% confidence intervals for the RLR parameters across 200 repeated trials, and calculate their empirical coverage probability (\ie, the proportion of trials in which
the estimated confidence interval contains the \true estimate)
as well as average interval width.
Table \ref{tab:bootstrap} shows the results when 15\% of each network have been
crawled.
We observe that the coverage probability for TS is higher than every other method across all datasets.%
\footnote{Notice that in some cases, even TS does not achieve the nominal
95\% coverage probability, possibly due to small errors in the parameter
estimates when the crawling proportion is relatively low.
Furthermore, regularization introduces an additional source of bias.
In particular, notice that the coverage probabilities of RLR-$\ell_1$ are significantly
lower than that of RLR-$\ell_2$ across all methods, which is due to the
shrinkage effect of the $\ell_1$ penalty that results in parameter estimates
with small values to be shrunk to exactly zero.}


\begin{figure*}[!htb]
\centering
\subfloat{
    \includegraphics[width=0.41\columnwidth,keepaspectratio]{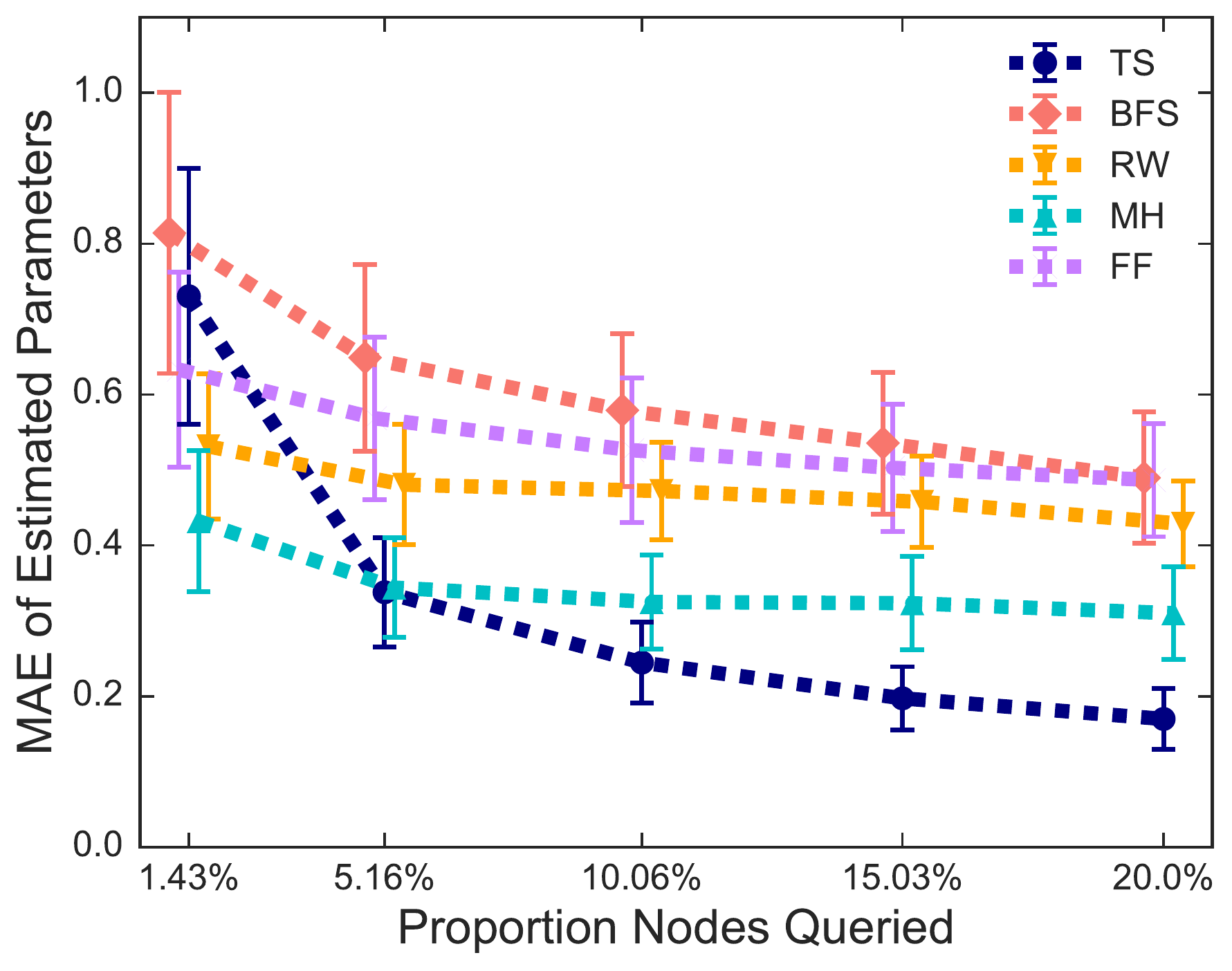}}
\hfill
\subfloat{
    \includegraphics[width=0.41\columnwidth,keepaspectratio]{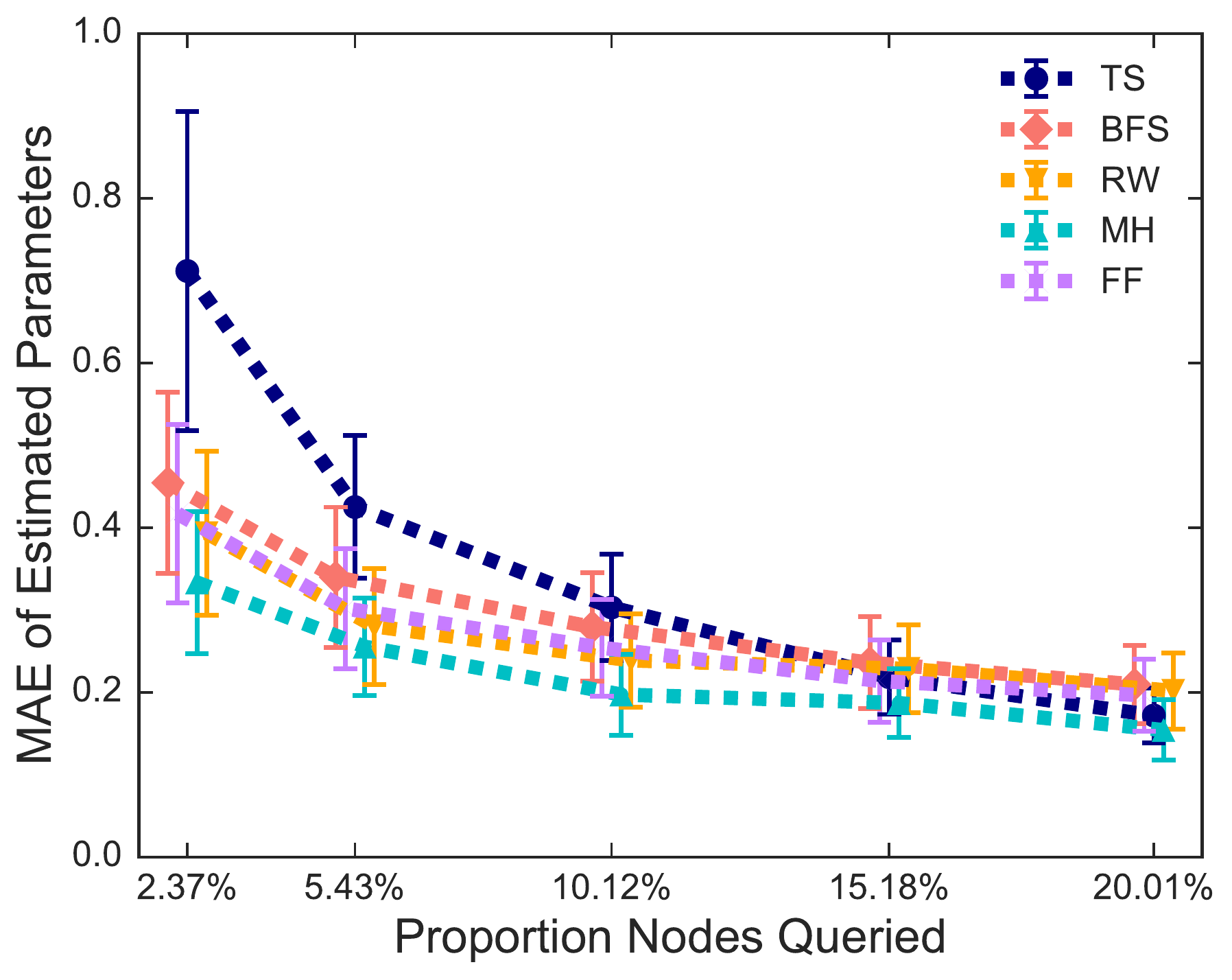}}
\hfill
\subfloat{
    \includegraphics[width=0.41\columnwidth,keepaspectratio]{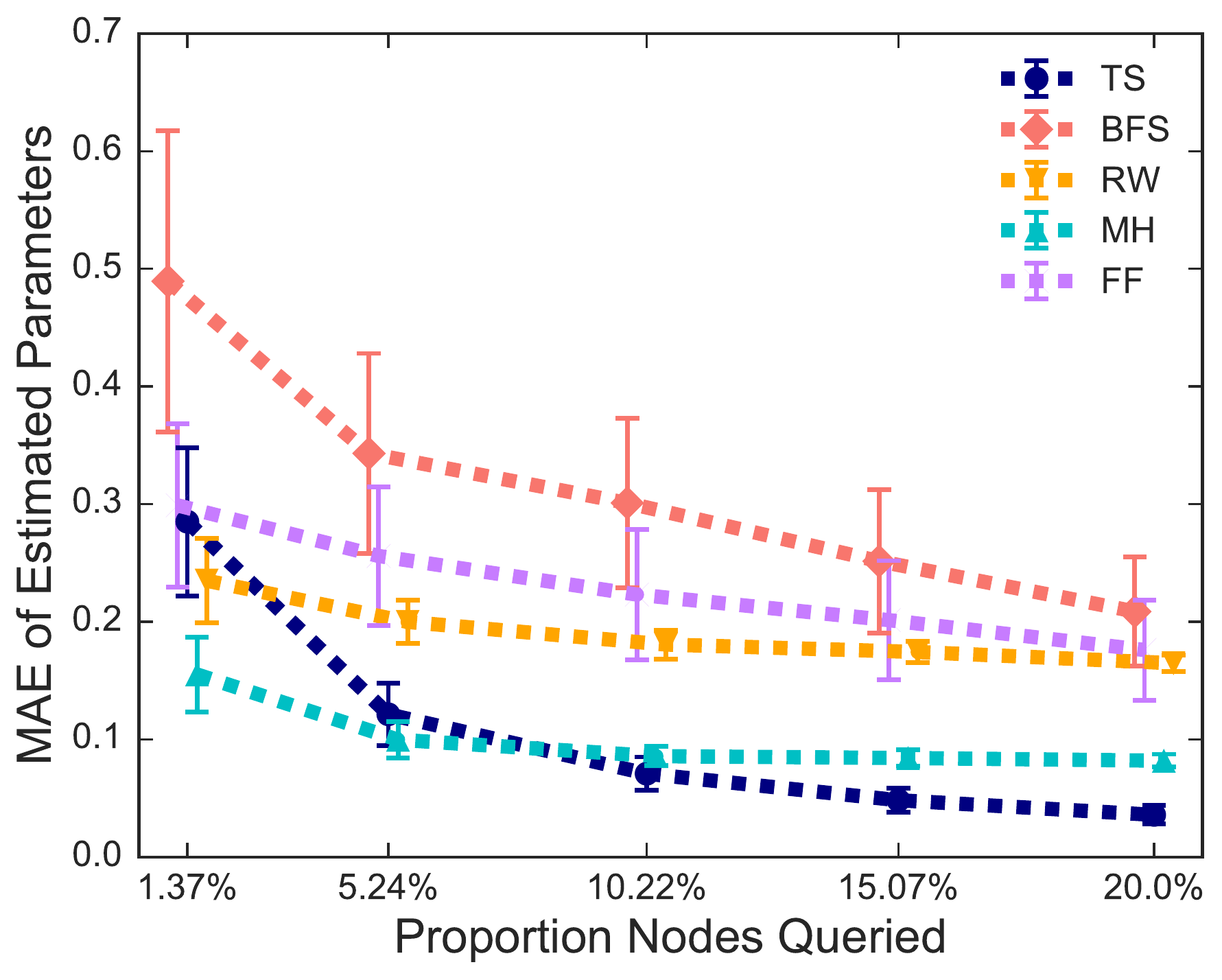}}
\hfill
\subfloat{
    \includegraphics[width=0.41\columnwidth,keepaspectratio]{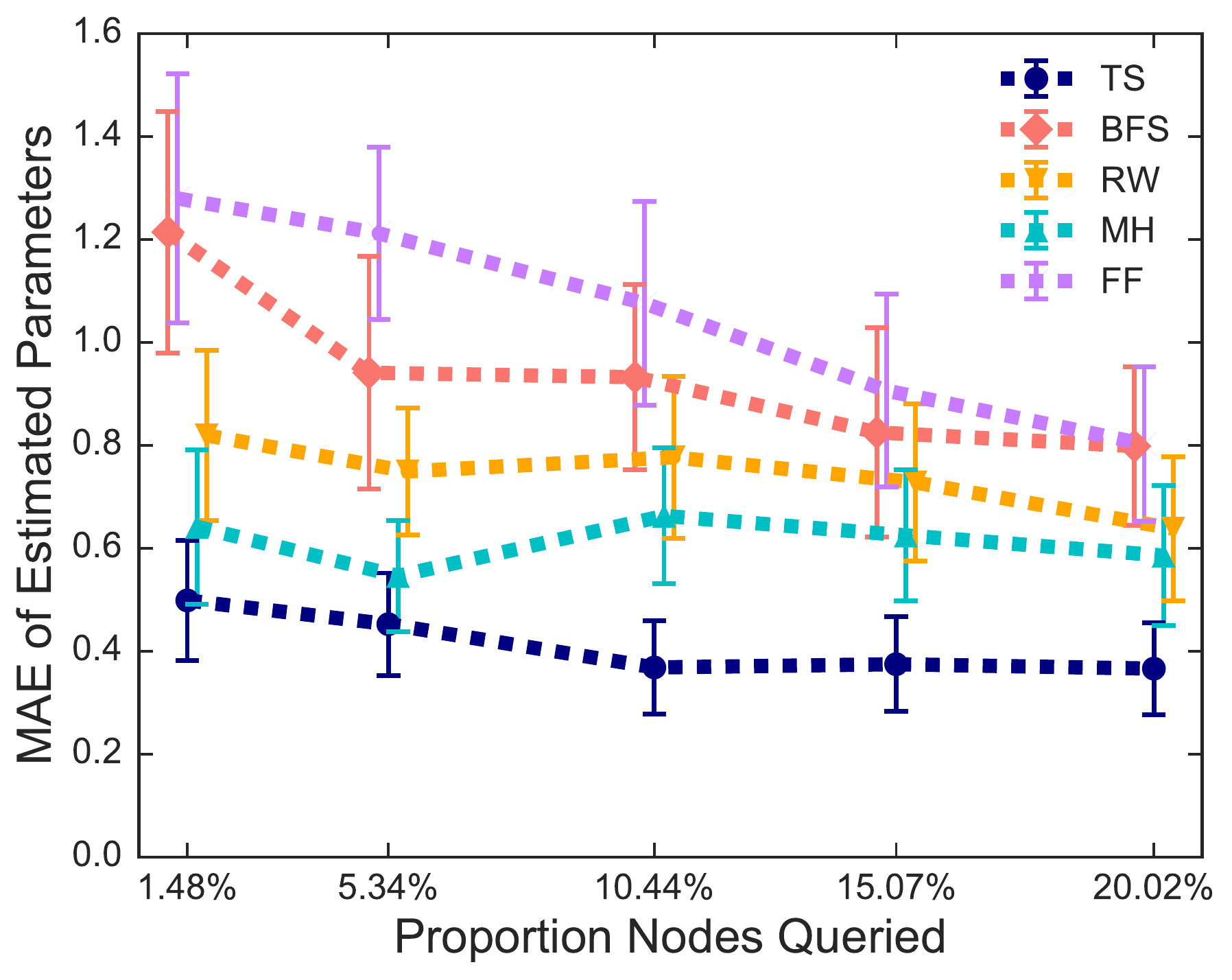}}
\subfloat{
    \includegraphics[width=0.41\columnwidth,keepaspectratio]{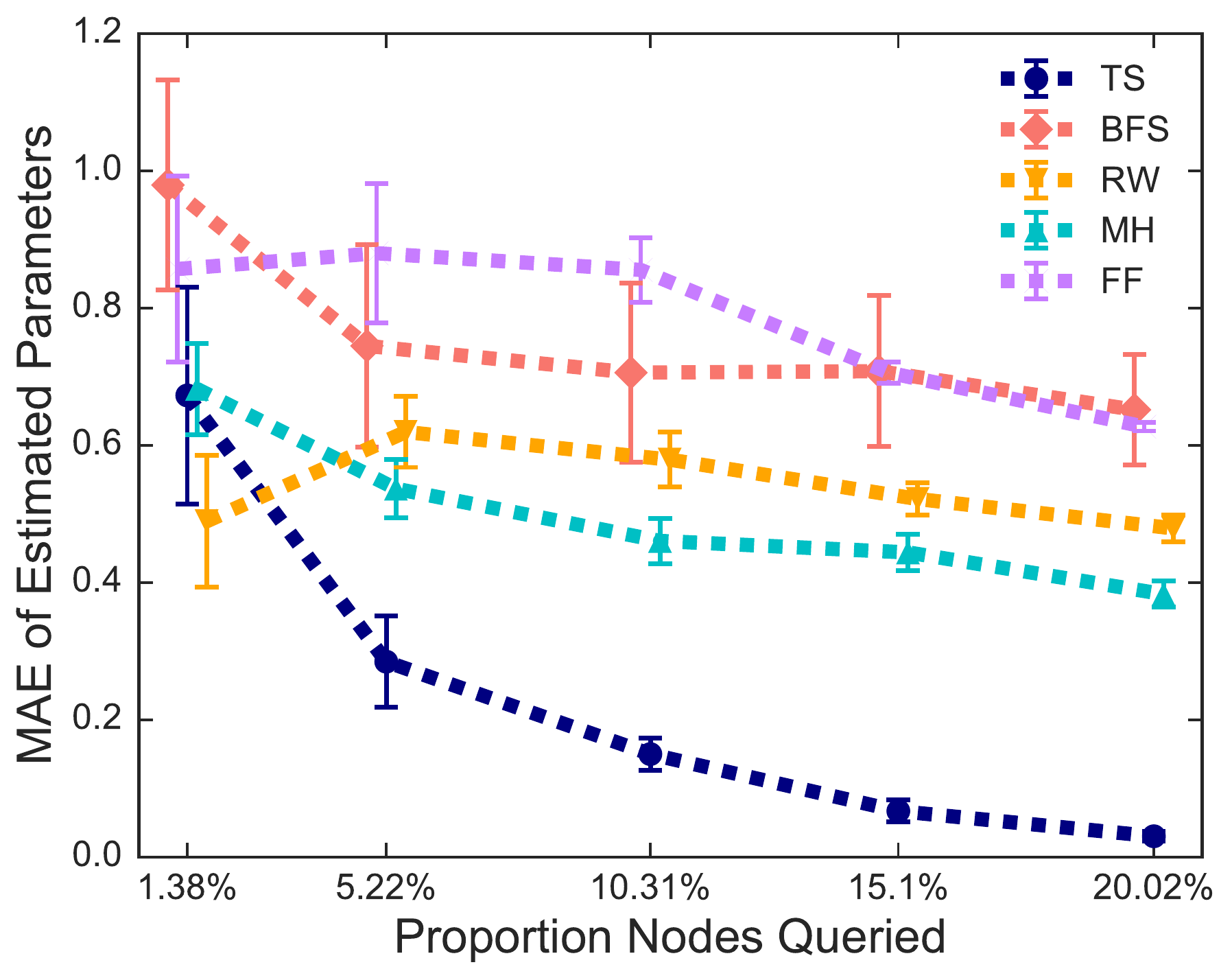}}
\vspace{-12pt}
\\
\addtocounter{subfigure}{-5}
\subfloat[Fri.-L. (Age); $\ell_1$]{
    \includegraphics[width=0.41\columnwidth,keepaspectratio]{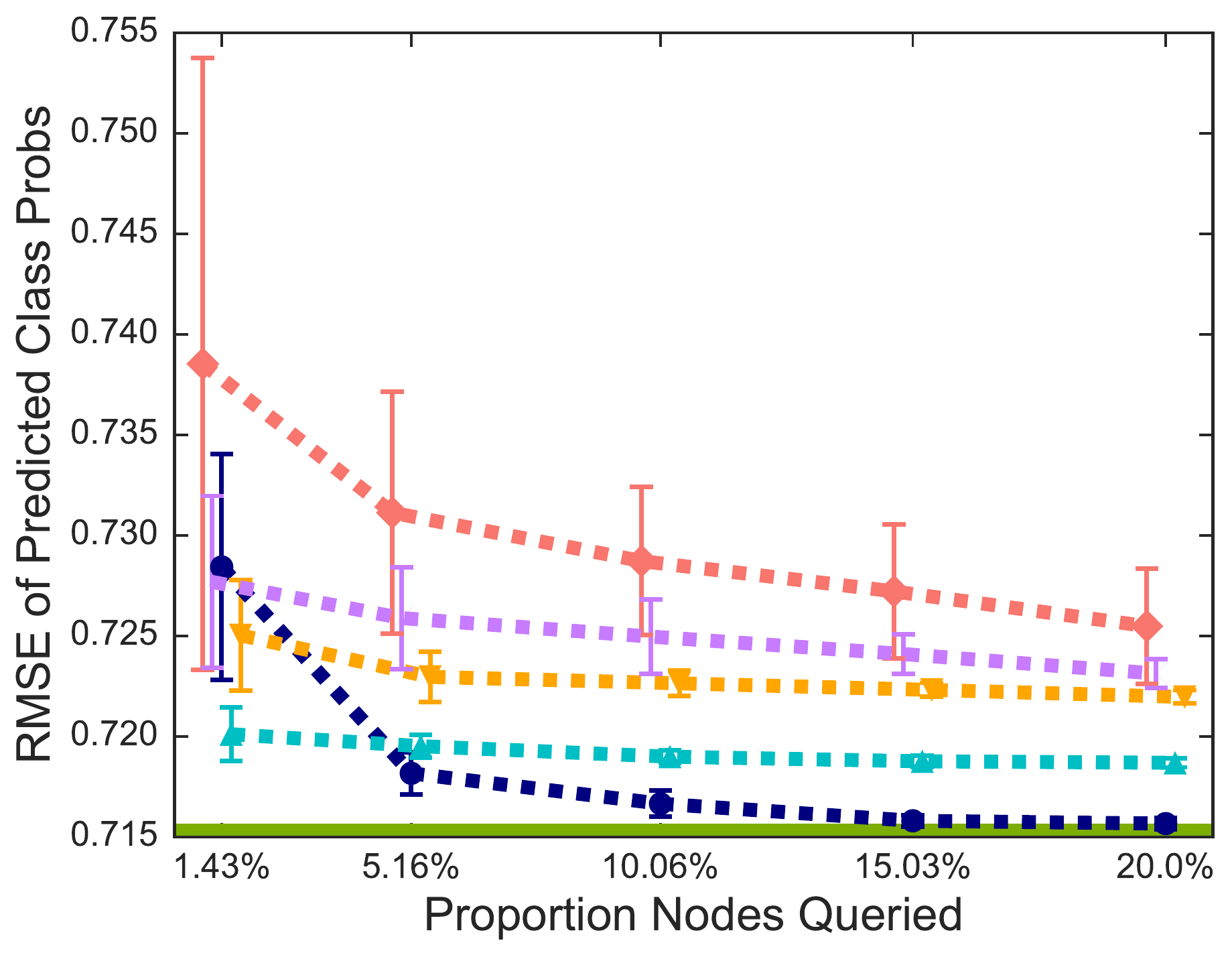}}
\hfill
\subfloat[Fri.-S. (Gender); $\ell_1$]{
    \includegraphics[width=0.41\columnwidth,keepaspectratio]{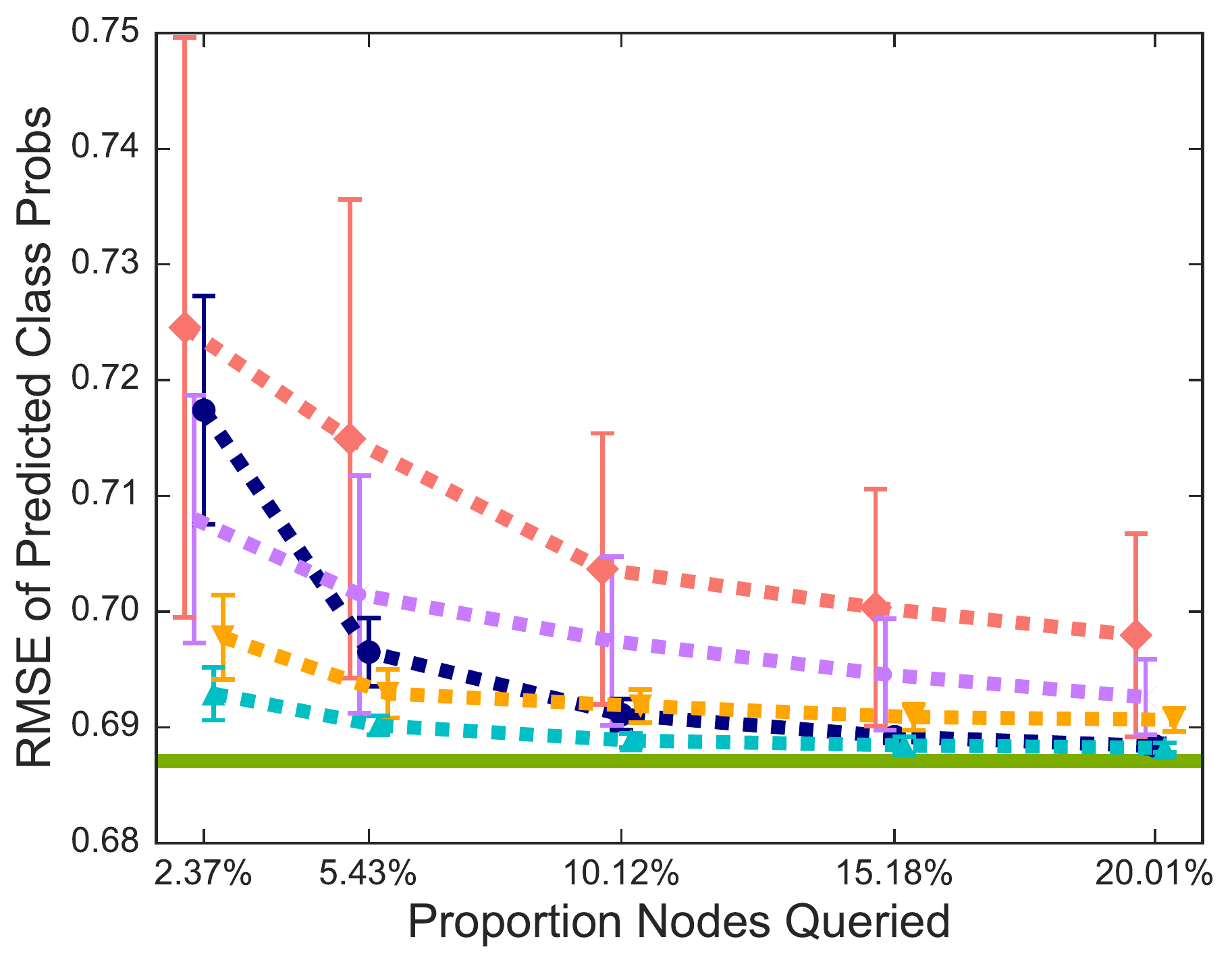}}
\hfill
\subfloat[Fri.-L. (Gender); $\ell_2$]{
    \includegraphics[width=0.41\columnwidth,keepaspectratio]{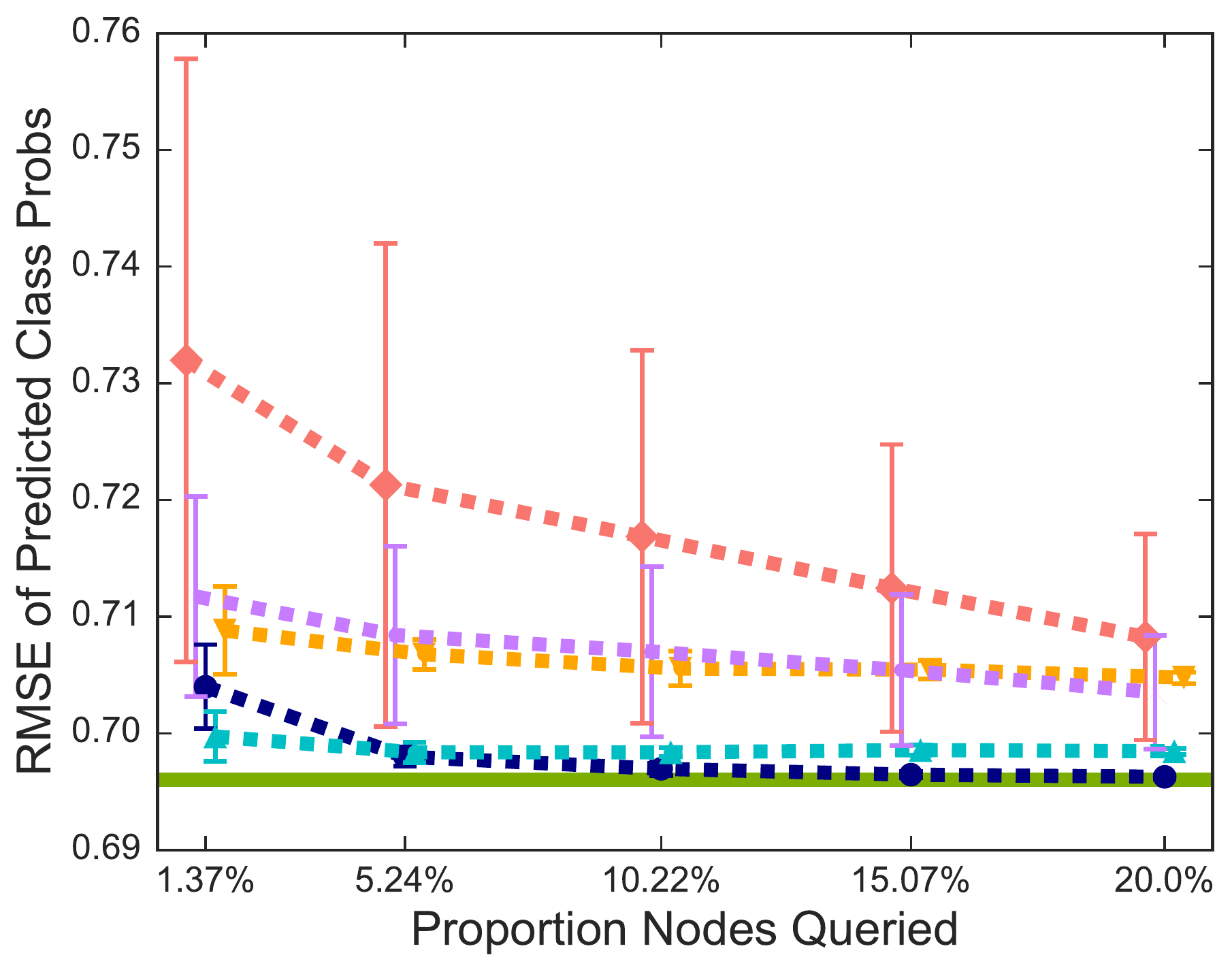}}
\hfill
\subfloat[Comm.; $\ell_1$]{
    \includegraphics[width=0.41\columnwidth,keepaspectratio]{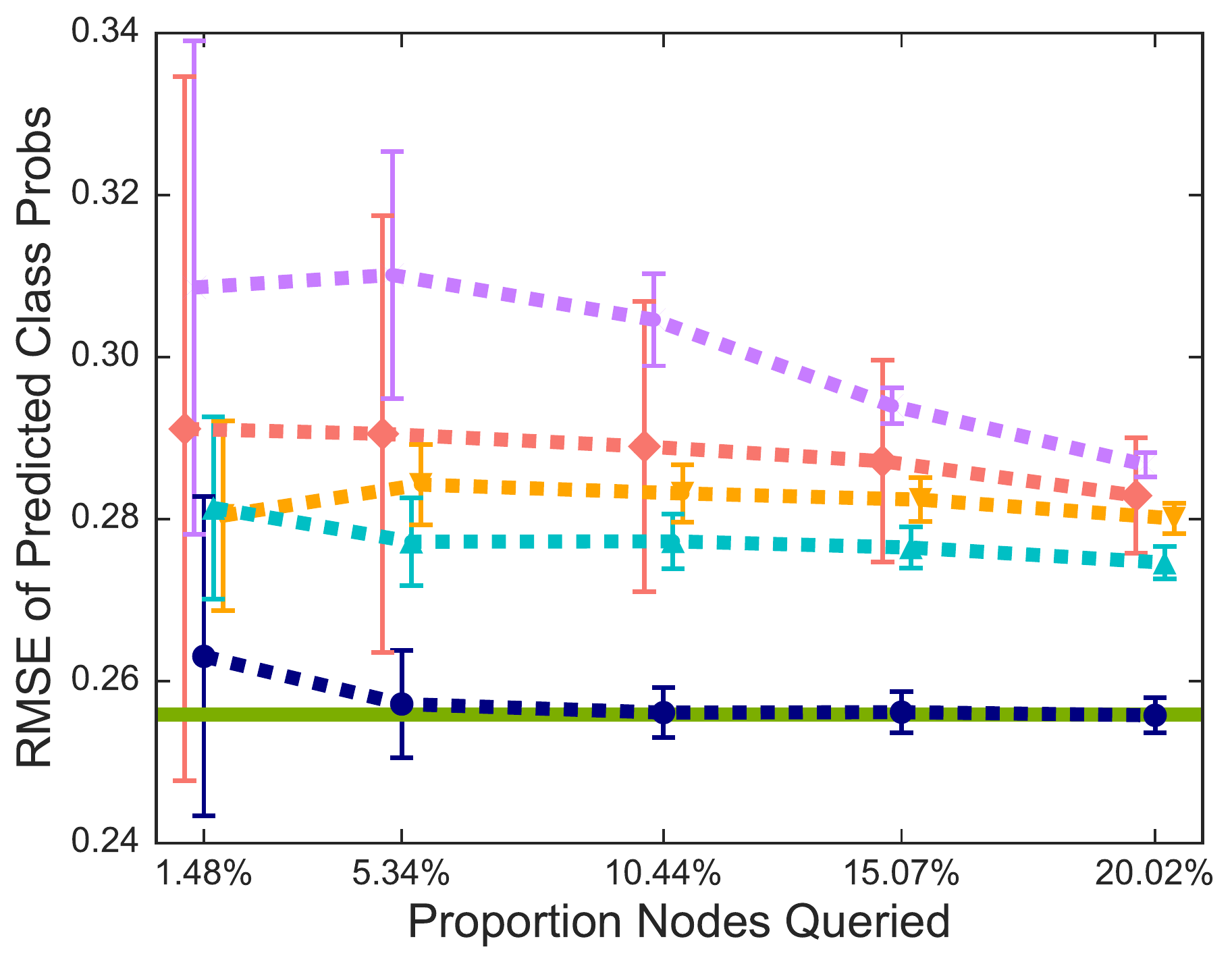}}
\hfill
\subfloat[Computers; $\ell_2$]{
    \includegraphics[width=0.41\columnwidth,keepaspectratio]{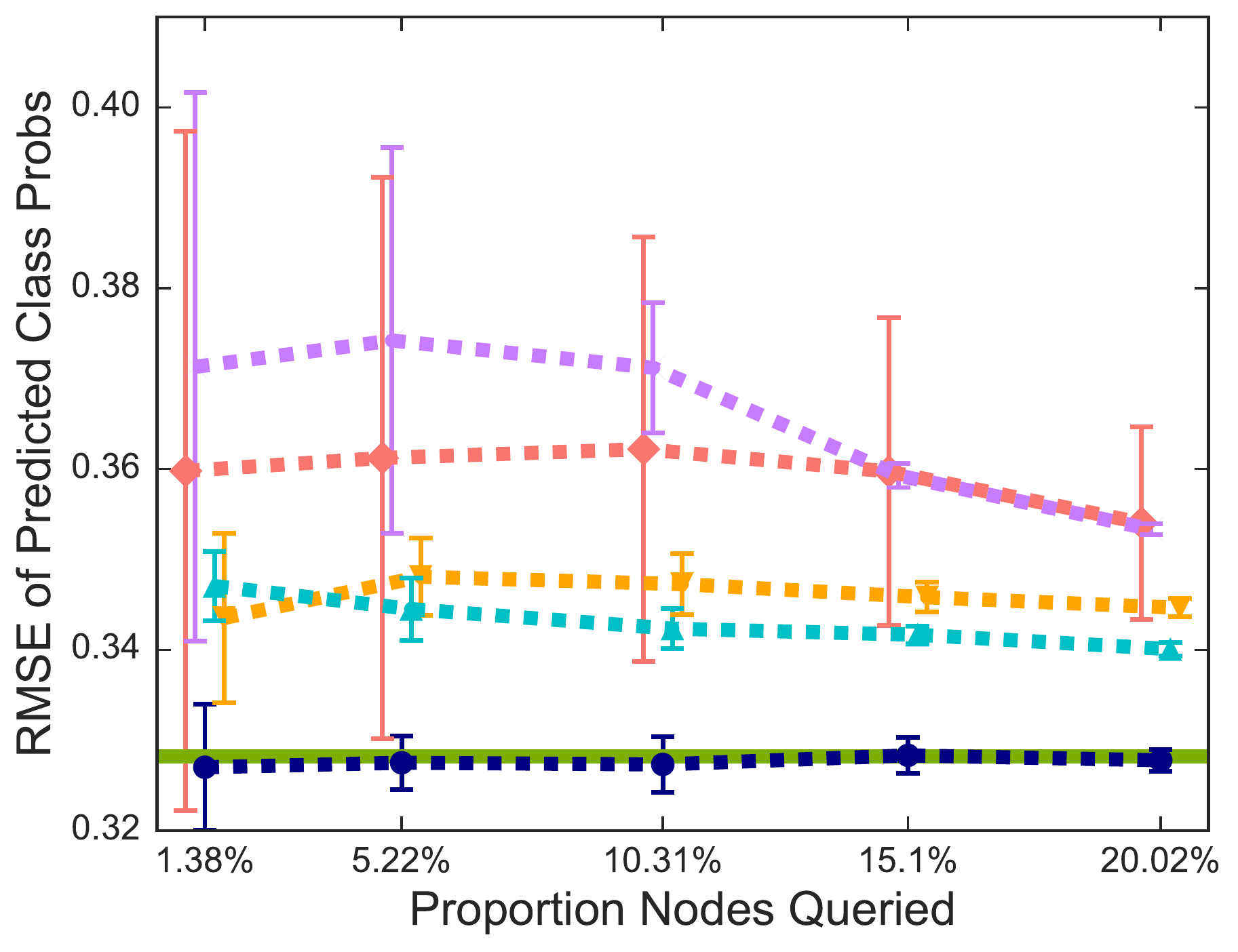}}
\vspace{-5pt}
\caption{MAE of estimated parameters (top row) and RMSE of predicted class probabilities (bottom row).}
\label{fig:estimation-rlr}
\end{figure*}

\begin{figure}[!htb]
\centering
\vspace{-5pt}
\subfloat[Fri.-L. (Age); $\ell_2$]{
    \includegraphics[width=0.49\columnwidth,keepaspectratio]{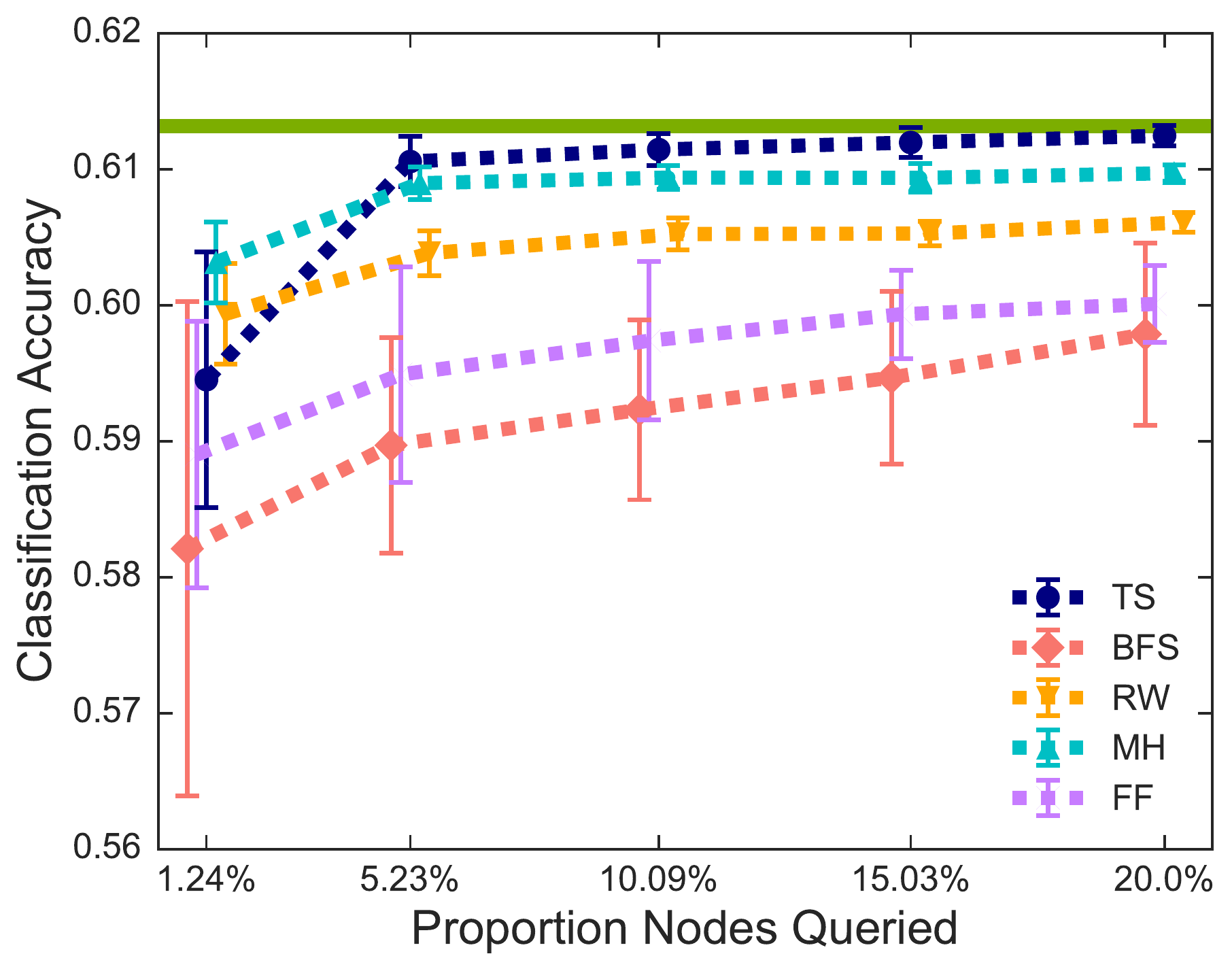}}
\subfloat[Fri.-L. (Gender); $\ell_2$]{
    \includegraphics[width=0.49\columnwidth,keepaspectratio]{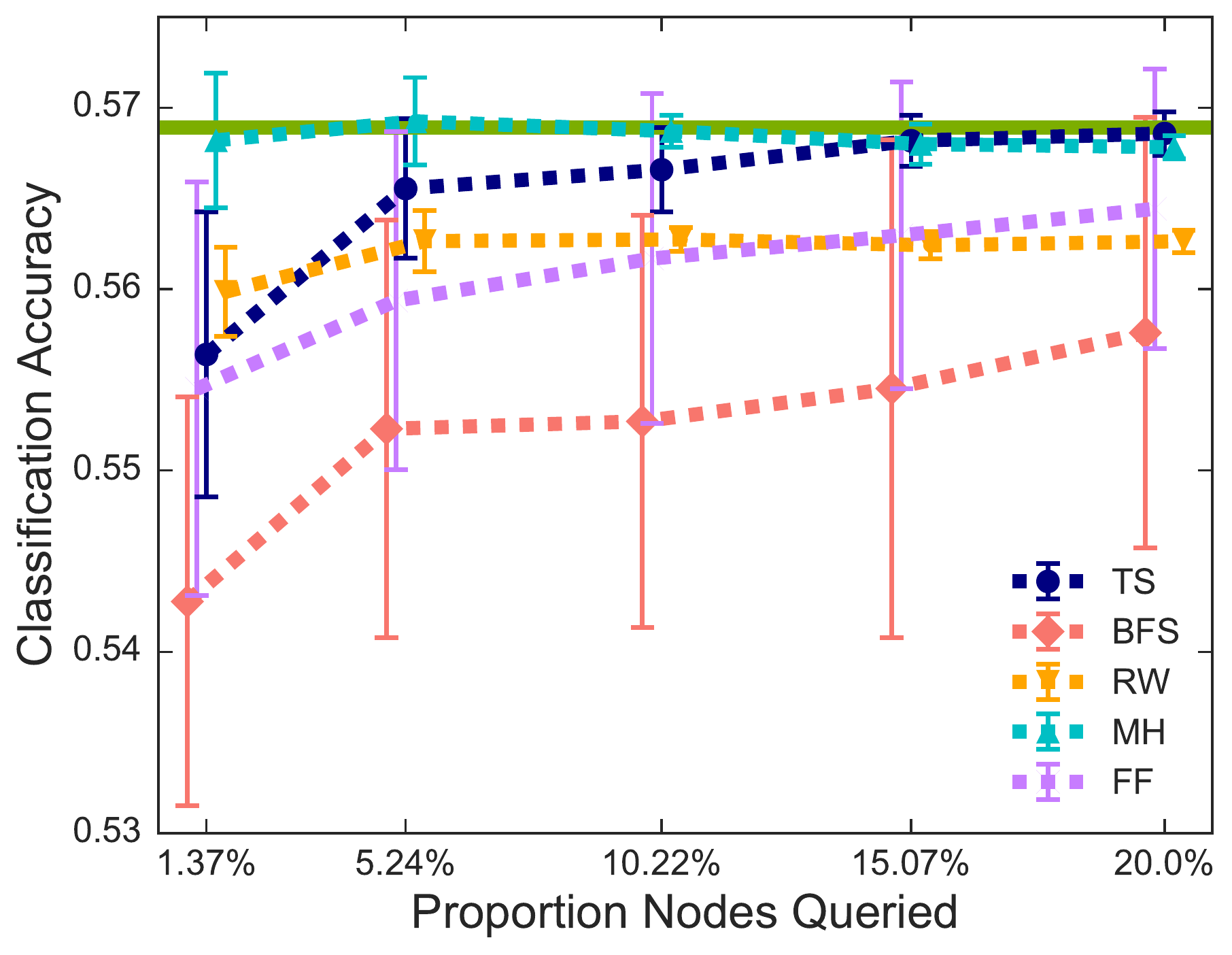}}
\hfil
\vspace{-5pt}
\caption{RLR classification accuracy.}
\label{fig:estimation-acc}
\vspace{-15pt}
\end{figure}

\begin{figure*}[!htb]
\centering
\vspace{-10pt}
\subfloat[Age in [16, 26) vs. proportion of neighbors in a relationship.]{
    \includegraphics[width=0.49\columnwidth,keepaspectratio]{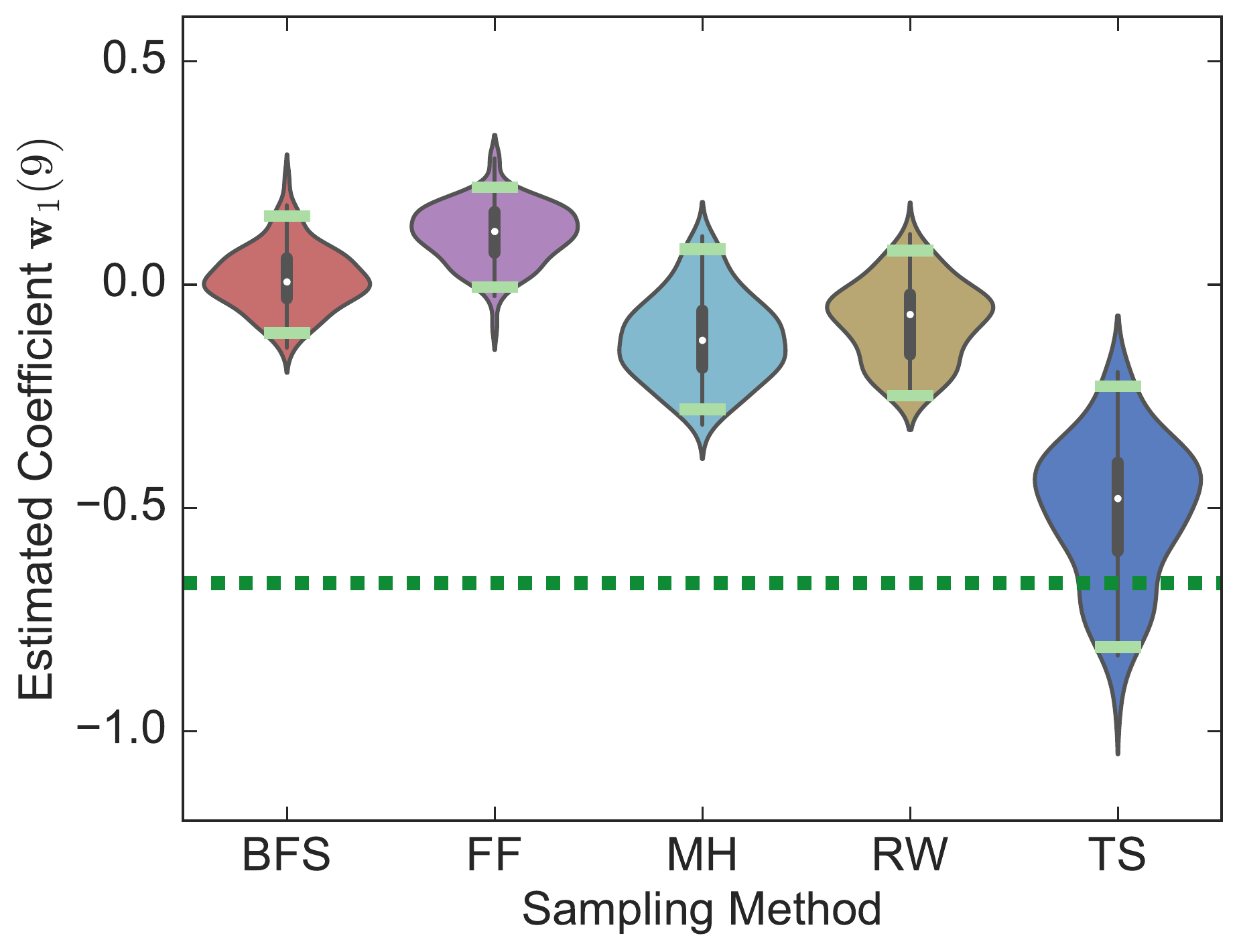}}
\hfil
\subfloat[Age in [16, 26) vs. proportion of married neighbors.]{
    \includegraphics[width=0.49\columnwidth,keepaspectratio]{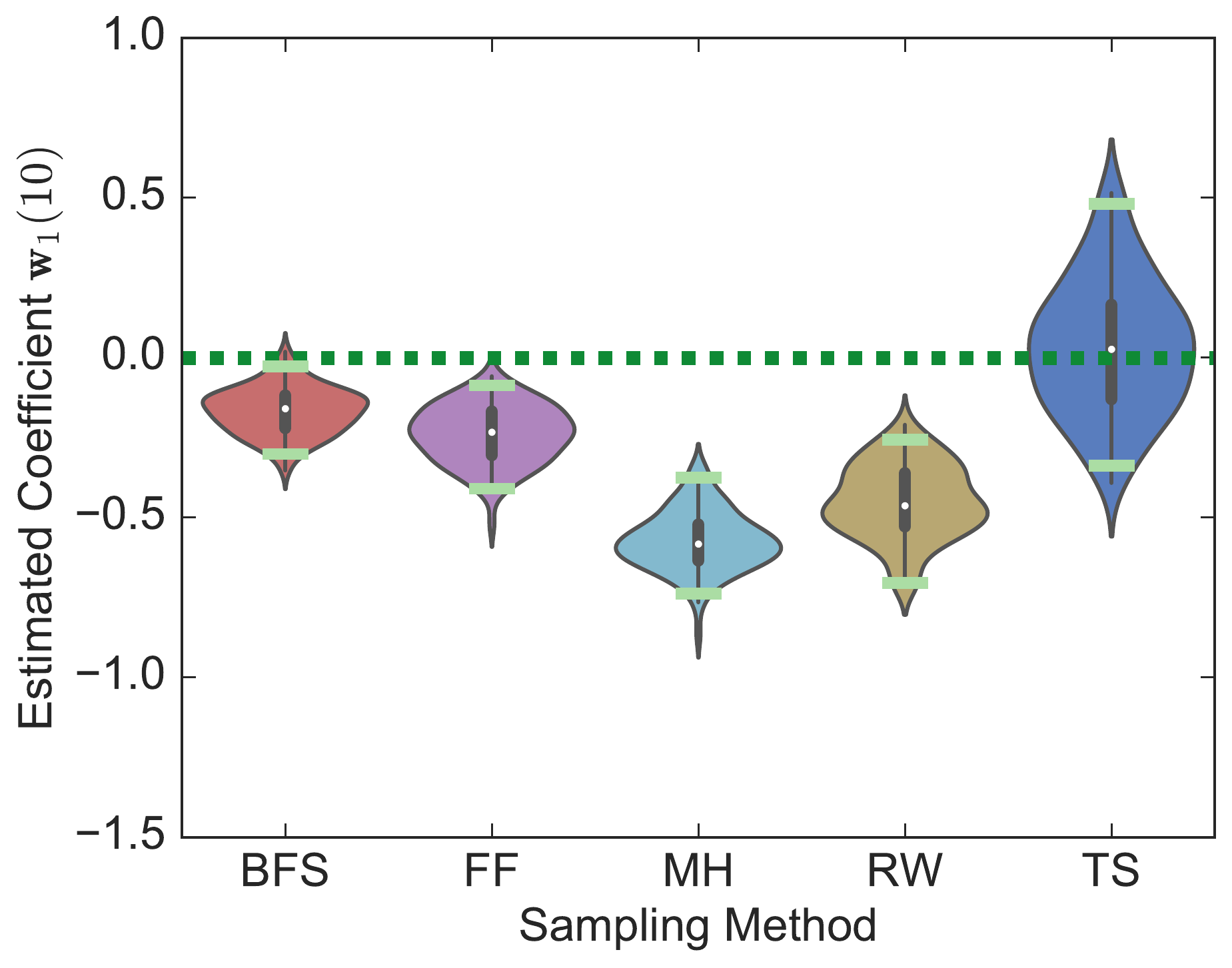}}
\hfil
\subfloat[Age in [26, 28) vs. proportion of neighbors in a relationship.]{
    \includegraphics[width=0.49\columnwidth,keepaspectratio]{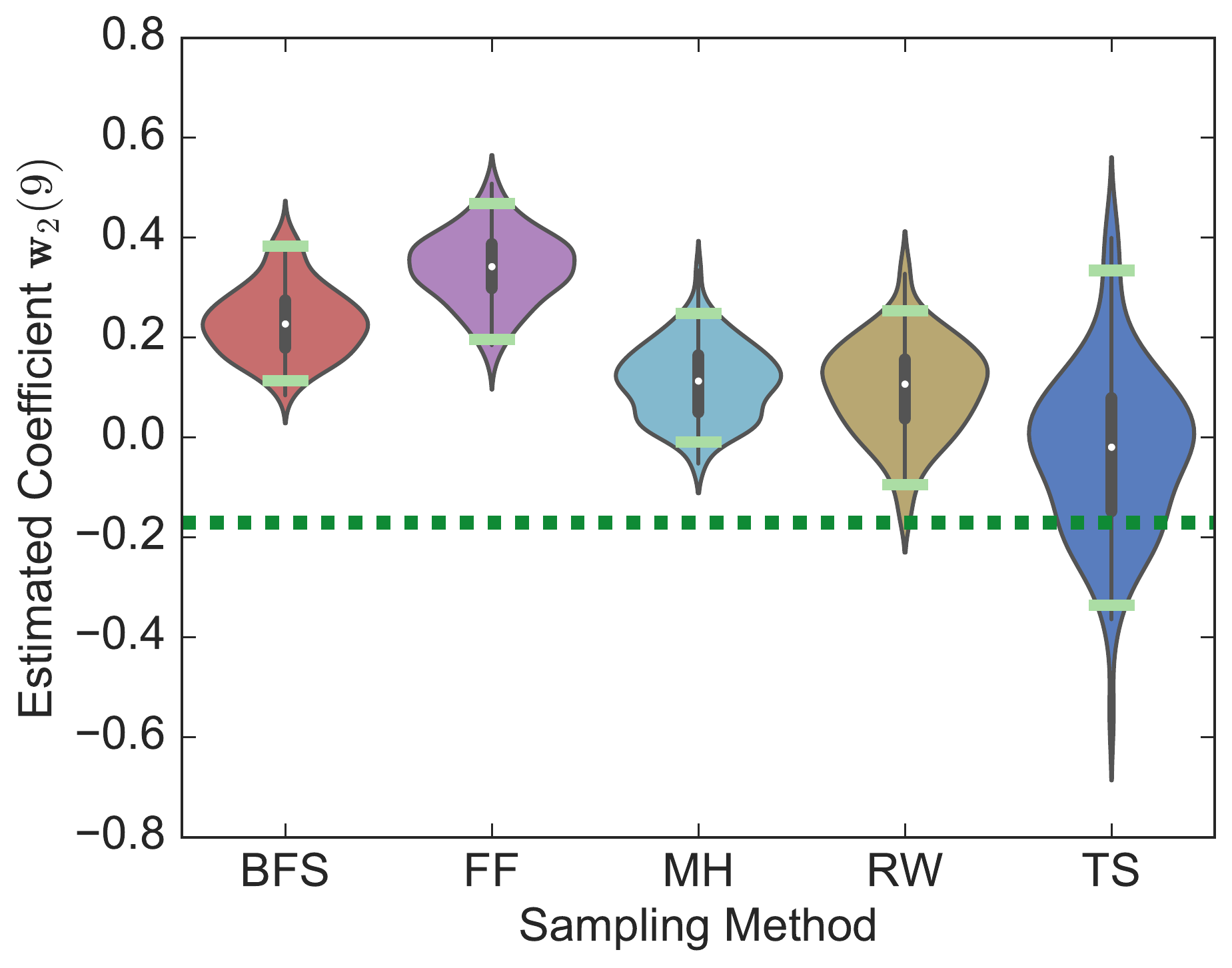}}
\hfil
\subfloat[Age in [32, 100) vs. proportion of neighbors with age in [16, 26).]{
    \includegraphics[width=0.50\columnwidth,keepaspectratio]{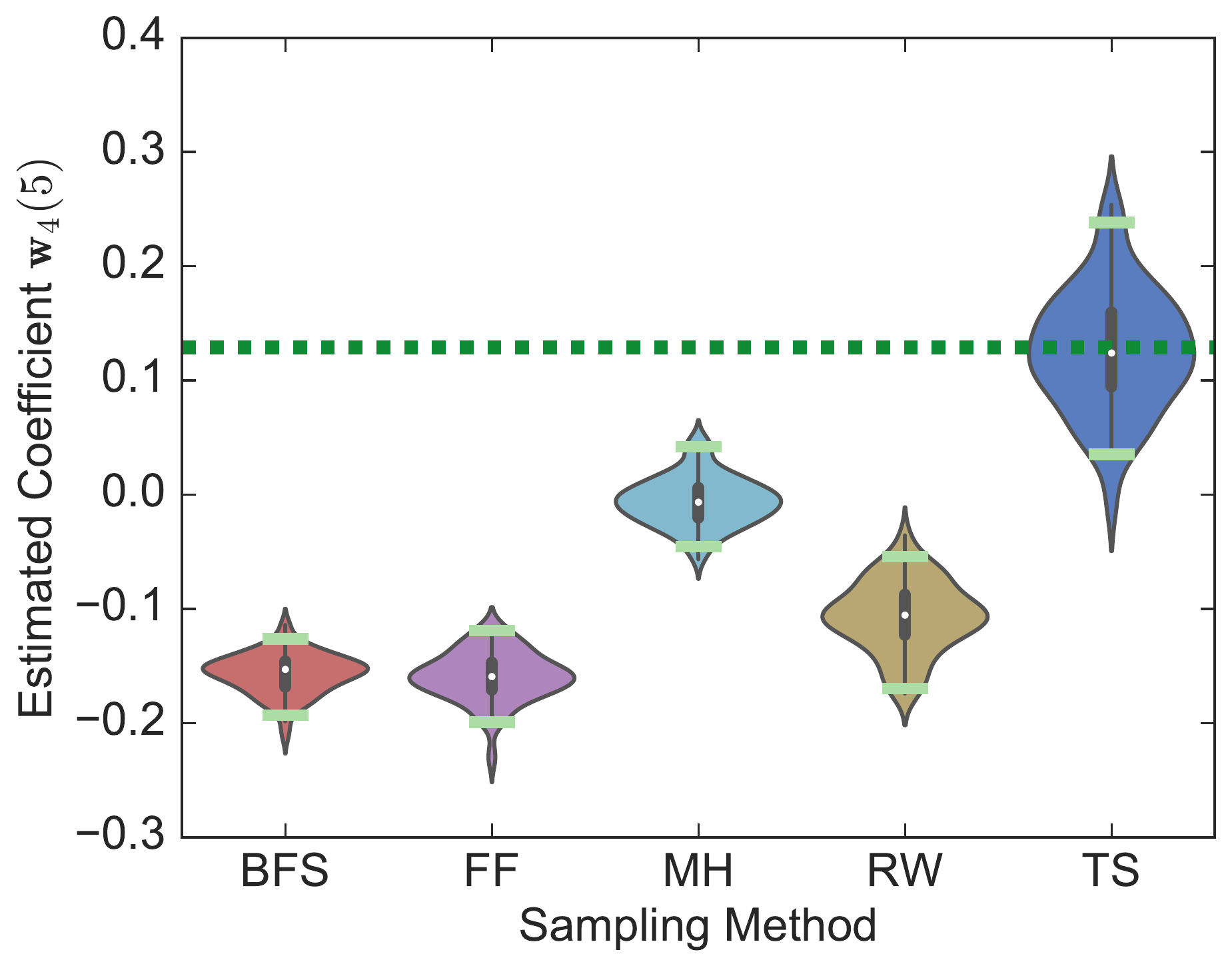}}
\vspace{-5pt}
\caption{Examples of estimated bootstrap sampling distributions using
RLR-$\ell_1$ on Friendster-Large (Age).}
\label{fig:variance-rlr}
\vspace{-10pt}
\end{figure*}


\begin{table*}[!htb]
\fontsize{8pt}{10pt}\selectfont
\centering
\caption{Coverage probability and average width of 95\%\,bootstrap confidence intervals on 15\%\,crawled network.}
\vspace{-8pt}
\label{tab:bootstrap}
\begin{tabular}{llccccc}
\toprule
& & \multicolumn{5}{c}{(Coverage Probability, Average Interval Width)} \\
\cmidrule(r){3-7}
Model & Dataset & BFS & FF & MH & RW & TS \\
\midrule
RLR-$\ell_1$ & Fri.-L. (Age) & (0.0842, 0.1849) & (0.0801, 0.1940) & (0.1887, 0.2439) & (0.1615, 0.2674) & (\textbf{0.7142}, 0.5313) \\
RLR-$\ell_1$ & Fri.-L. (Gender) & (0.0714, 0.0757) & (0.0741, 0.0986) & (0.1278, 0.0979) & (0.1019, 0.1177) & (\textbf{0.5239}, 0.2167) \\
RLR-$\ell_1$ & Fri.-L. (Status) & (0.1946, 0.2636) & (0.1470, 0.2500) & (0.4548, 0.3876) & (0.3179, 0.3577) & (\textbf{0.7479}, 0.7591) \\
RLR-$\ell_1$ & Fri.-S. (Gender) & (0.4825, 0.2523) & (0.4810, 0.2886) & (0.7307, 0.3941) & (0.6346, 0.4384) & (\textbf{0.7443}, 0.8247) \\
RLR-$\ell_1$ & Comm. & (0.0000, 0.1475) & (0.0000, 0.1205) & (0.0000, 0.1817) & (0.0000, 0.1888) & (\textbf{0.6240}, 0.4346) \\
RLR-$\ell_1$ & Computers & (0.0000, 0.1008) & (0.0000, 0.1034) & (0.0000, 0.1268) & (0.0000, 0.1206) & (\textbf{0.2500}, 0.3595) \\
RLR-$\ell_2$ & Fri.-L. (Age) & (0.1393, 0.1739) & (0.1912, 0.1897) & (0.3687, 0.2457) & (0.3392, 0.2590) & (\textbf{0.8321}, 0.5160) \\
RLR-$\ell_2$ & Fri.-L. (Gender) & (0.1429, 0.0824) & (0.1429, 0.0863) & (0.3641, 0.0977) & (0.3552, 0.1045) & (\textbf{0.9996}, 0.2282) \\
RLR-$\ell_2$ & Fri.-L. (Status) & (0.3090, 0.2576) & (0.3055, 0.2765) & (0.5339, 0.3391) & (0.4078, 0.3331) & (\textbf{0.8707}, 0.7261) \\
RLR-$\ell_2$ & Fri.-S. (Gender) & (0.5508, 0.2606) & (0.6001, 0.2858) & (0.6722, 0.3909) & (0.7502, 0.4220) & (\textbf{0.9568}, 0.6578) \\
RLR-$\ell_2$ & Comm. & (0.0000, 0.1464) & (0.0000, 0.1556) & (0.0000, 0.2026) & (0.0000, 0.2100) & (\textbf{1.0000}, 0.5279) \\
RLR-$\ell_2$ & Computers & (0.0000, 0.1099) & (0.0000, 0.1296) & (0.0000, 0.1522) & (0.0000, 0.1462) & (\textbf{1.0000}, 0.3480) \\
\bottomrule
\end{tabular}
\vspace{-10pt}
\end{table*}


\section{Conclusion and Future Work}\label{sec:conclusions}

In this work, we developed a stochastic gradient descent method for learning
relational logistic regression models from large-scale networks via partial
crawling.
We proved that the proposed method yields unbiased estimates of the
log-likelihood and its gradients over the full graph, and demonstrated how to
construct confidence intervals of the model parameters.
Our experiments showed that the proposed method produces more
accurate parameter estimates and confidence intervals compared to naively
learning models from data collected by existing crawlers.

\subsubsection{Estimation of Templated Relational Models}

One line of future work would be to extend our stochastic estimation
method to the family of \emph{templated relational models}, such as
\emph{relational Markov networks} (RMNs) \citep{TasAbbKol02},
\emph{Markov logic networks} (MLNs) \citep{DomingosRichardson04},
and \emph{relational dependency networks} (RDNs) \citep{NevJen07}
under the same crawl-based scenario examined in this work.

Notice that RLR can be viewed as a direct analog of MLNs where the weighted
logic formulas are used to define conditional instead of joint probabilities.
More generally, when the cliques in RMNs and MLNs are defined over connected subgraphs up to size three, we expect that the full log-likelihood can be
unbiasedly estimated by extending our proposed methodology.
However, learning RMNs and MLNs that contain higher-order clique structures
would be more challenging.

For RDNs, the estimation procedure depends on the specific local conditional model one employs.
Our previous work \citep{YanRibNev17} can be used to learn RDNs when the RBC is used
to model the local conditional probability distributions (CPDs),
and the current work has addressed the case when RLR is used as the local model
component.
Another possible choice for the local model consist of
\emph{relational probability trees} (RPTs) \citep{NevJenFriHay03}.
Similar to RLR, RPTs also utilize aggregation functions to construct node-centric features.
While these aggregated features can be stochastically estimated in the same way
as in RLR, learning the full tree structure becomes difficult due to the greedy partitioning procedure involved.

\subsubsection{Impact of Sampling on Collective Inference}

Finally, we note that in this work we utilize {\em non}-collective inference on the full graph---that is, when predicting an unlabeled node, we do not utilize the predictions made for its unlabeled neighbors, and instead treat their class labels as missing.
Although collective inference has been shown to improve classification accuracy, it also introduces \emph{inference error} \citep{XiangNeville11}.
In this work, we have mainly focused on examining
the impact of crawling on \emph{learning}, but future work should investigate
the more complex interplay between sampling, learning, and inference.

\section*{Acknowledgements}

This research is supported by NSF under contract numbers IIS-1149789,
IIS-1546488, IIS-1618690, and by the Army Research Laboratory under Cooperative Agreement Number W911NF-09-2-0053.

\appendix

\section{Proofs of Theorems}

\begin{proof}\textbf{of Theorem~\ref{thm:rlr}.}
It will be convenient to define the function
\displayfont{
\begin{equation}\label{eq:func}
  f(u,v) \defeq
    \begin{cases}
      \nicefrac{g(v)}{d_v} & \mbox{if $v\in V\backslash\Scal$;} \\
      0                & \mbox{if $v\in\Scal$.}
    \end{cases}
\end{equation}
}%
Recall that our goal is to estimate
\displayfont{
\begin{equation}\label{eq:full-est}
  \Lcal = \sum_{v\in V} g(v)
    = \sum_{v\in V \backslash \Scal} g(v) + \sum_{v\in \Scal} g(v)
    \defeq \Lcal^t + \Lcal^s,
\end{equation}
}%
where $\Lcal^s$ can be computed explicitly.
Notice that%
\footnote{Recall that for undirected graphs, the edge-set $E$ contains both
copies of each edge---\ie, $(u,v)\in E$ if and only if $(v,u)\in E$.}
\displayfont{
\begin{equation}\label{eq:estimand}
  \Lcal^t = \sum_{v\in V \backslash \Scal} g(v) = \sum_{(u,v)\in E} f(u, v).
\end{equation}
}%
Given the random walk tours
\displayfont{
$\Dcal_m(\Scal) = \{(v^{(k)}_1,\,\cdots,v^{(k)}_{\xi_k})\}_{k=1}^m$,
}%
we estimate $\Lcal^t$ using
$\widehat{\Lcal}^t
  \defeq \frac{1}{m} \sum_{k=1}^m \widehat{\Lcal}^t_k$,
where
\displayfont{
\begin{align}\label{eq:edge-est}
\widehat{\Lcal}^t_k
  & \defeq d_\Scal \sum_{t=2}^{\xi_k} f(v^{(k)}_{t-1}, v^{(k)}_t)
    = d_\Scal \sum_{t=2}^{\xi_k-1} \frac{g(v^{(k)}_t)}{d_{v^{(k)}_t}} \, .
\end{align}
}%
Here,
$d_{\Scal} = |((\Scal\times V) \cap E) \backslash (\Scal\times\Scal)|$ denotes
the total number of outgoing edges from the seed nodes,
and the last equality follows from Eq.\,\eqref{eq:func}.
Plugging the estimate into Eq.\,\eqref{eq:full-est},
we arrive at the full estimate of $\Lcal$:
\displayfont{
\begin{equation}\label{eq:loglik-est}
\widehat{\Lcal}
\defeq \widehat{\Lcal}^t + \Lcal^s
  = \frac{d_{\Scal}}{m}  \sum_{k=1}^m \sum_{t=2}^{\xi_k-1} \frac{g(v^{(k)}_t)}{d_{v^{(k)}_t}}
  + \sum_{v \in \cS} g(v) \, .
\end{equation}
}%
Moreover, the gradients of $\Lcal$ can be estimated by
\displayfont{
\begin{equation}\label{eq:loglikgrad-est}
\nabla_{\bw_j} \widehat{\Lcal} (\bw_1,\ldots,\bw_H)
= \frac{d_{\Scal} }{m} \sum_{k=1}^m  
\sum_{t=2}^{\xi_k-1} \frac{g'_j(v^{(k)}_t)}{d_{v^{(k)}_t}}
+ \sum_{v \in \cS} g'_j(v)\, .
\end{equation}
}%
To show that the estimates of Eqs.\,\eqref{eq:loglik-est} and
\eqref{eq:loglik-est} are unbiased, it suffices to show that for all $k$,
$\widehat{\Lcal}^t_k$ is an unbiased estimate of~$\Lcal^t$,
since the sample average $\widehat{\Lcal}^t$ will also be unbiased, and
the gradient is a linear operator.
More formally, we prove that for any $k=1,\,\cdots,m$, we have
%
\displayfont{
\begin{equation}
  \Ex{d_\Scal \sum_{t=2}^{\xi_k} f(v^{(k)}_{t-1}, v^{(k)}_t)} =
  \sum_{(u,v)\in E} f(u,v).
\end{equation}
}%
%
%
Notice that the random walk tour sampling algorithm 
(\cf Section~\ref{sec:background_sampling})
is equivalent to a conventional random walk conducted on
a virtual multi-graph $\widetilde{G}$ formed by treating all the seed nodes in $\Scal$ as a
single ``super-node'' while retaining all outgoing edges.
Thus, the degree of the super-node is $d_\Scal$.
The sampling process, viewed as a random walk on~$\tilde{G}$, constitutes a
\emph{renewal process} in which a renewal occurs when the walk returns to the
super-node (thereby completing a tour).
%
For the $k$-th tour
\displayfont{$(v^{(k)}_1,\,\cdots,v^{(k)}_{\xi_k})$},
define its \emph{reward} as
\displayfont{
$$Y_k \defeq \sum_{t=2}^{\xi_k} f(v^{(k)}_{t-1}, v^{(k)}_t)\,
    \ind{v^{(k)}_{t-1}=u, v^{(k)}_{t-1}=v},
$$
}%
where $u$ and $v$ are two adjacent vertices in $V \backslash \Scal$.
By the Markov property, both the 
tour lengths $\{\xi_k\}_{k=1}^m$ and
the rewards $\{Y_k\}_{k=1}^m$ are \iid sequences.
Let
\displayfont{
$$ N(i) \defeq \min\{n: \sum_{k=1}^n \xi_k \le i\} $$
}%
be the number of renewals (visits to the super-node) up to sampling the $i$-th node in the random walk.
Then the \emph{renewal reward theorem} %
 \cite[Chapter 3, Theorem 4.2]{Bremaud99} implies that
\displayfont{
\begin{equation}\label{eq:renewal-reward}
  \lim_{i\to\infty} \frac{\sum_{k=1}^{N(i)} Y_k}{i} = \frac{\Ex{Y_1}}{\Ex{\xi_1}}.
\end{equation}
}%
Let $\widetilde{E} = E \backslash (\Scal\times \Scal)$ be the edge-set of
$\widetilde{G}$.
The stationary probability of a random walk
on $\widetilde{G}$ are given by $\pi_v = \nicefrac{d_v}{|\widetilde{E}|}$,
and the transition probability from $u$
to $v$ is given by $P_{uv} = \nicefrac{1}{d_u}$.
Therefore,
\displayfont{
\begin{equation}\label{eq:asymp-freq}
  \lim_{i\to\infty} \frac{\sum_{k=1}^{N(i)} Y_k}{i} = \
    \pi_{u}\, P_{uv} f(u,v) = \frac{1}{|\widetilde{E}|} f(u,v),
\end{equation}
}%
and the mean recurrence time is
$\Ex{\xi_1} = 1/\pi_{\Scal} = |\widetilde{E}| / d_{\Scal}$.
Joining Eq.\,\eqref{eq:renewal-reward} and Eq.\,\eqref{eq:asymp-freq},
and multiplying by $|\widetilde{E}|$ on both sides, we have that
\displayfont{
\begin{equation*}
 d_\Scal\, \Ex{\sum_{t=2}^{\xi_1} f(v^{(1)}_{t-1}, v^{(1)}_t)\,
      \ind{v^{(1)}_{t-1}=u, v^{(1)}_{t}=v}}
  = f(u,v),
\end{equation*}
}%
Finally, taking the sum over all $(u,v)\in E$ yields Eq.\,\eqref{eq:edge-est},
which concludes our proof.
\end{proof}


\bibliographystyle{aaai}

\fontsize{9pt}{10pt}\selectfont
\bibliography{paper}

\end{document}